\documentclass{article}
\usepackage[accepted]{icml2020}
\usepackage{microtype}
\usepackage{graphicx}
\usepackage{subfigure}
\usepackage{booktabs} 

\usepackage[colorlinks,
           linkcolor=red,
           citecolor=blue,
           urlcolor=magenta,
           linktocpage,
           plainpages=false,draft]{hyperref}
\usepackage{relsize}
\usepackage{url}            
\usepackage{booktabs}       
\usepackage{amsfonts}       
\usepackage{nicefrac}       
\usepackage{microtype}      
\usepackage{amsmath}
\usepackage{amsthm}
\usepackage{amssymb}
\usepackage{thm-restate}
\usepackage{subfigure}

\usepackage[colorinlistoftodos, textwidth=30mm, shadow, textsize=small]{todonotes}
\definecolor{babyblue}{rgb}{0.54, 0.81, 0.94}
\definecolor{citrine}{rgb}{0.89, 0.82, 0.04}
\definecolor{misocolor}{rgb}{0.16,0.27,0.86}

\usepackage{float}

\definecolor{blued}{RGB}{70,197,221}
\definecolor{pearOne}{HTML}{2C3E50}
\definecolor{pearTwo}{HTML}{A9CF54}
\definecolor{pearTwoT}{HTML}{C2895B}
\definecolor{pearThree}{HTML}{E74C3C}
\colorlet{titleTh}{pearOne}
\colorlet{bull}{pearTwo}
\definecolor{pearcomp}{HTML}{B97E29}
\definecolor{pearFour}{HTML}{588F27}
\definecolor{pearFith}{HTML}{ECF0F1}
\definecolor{pearDark}{HTML}{2980B9}
\definecolor{pearDarker}{HTML}{1D2DEC}
\usepackage{hyperref}       
\hypersetup{
	colorlinks,
	citecolor=pearDark,
	linkcolor=pearThree,
	breaklinks=true,
	urlcolor=pearDarker}

\definecolor{misocolor}{rgb}{0.16,0.27,0.86}

\DeclareMathOperator*{\argmax}{arg\,max}
\DeclareMathOperator*{\argsup}{arg\,sup}

\newtheorem{assumption}{Assumption}

\newtheorem{theorem}{Theorem}
\newtheorem{proposition}{Proposition}

\newtheorem{corollary}{Corollary}
\newtheorem{remark}{Remark}

\let\originalleft\left
\let\originalright\right
\renewcommand{\left}{\mathopen{}\mathclose\bgroup\originalleft}
\renewcommand{\right}{\aftergroup\egroup\originalright}

\newcommand{\CommaBin}{\mathbin{\raisebox{0.5ex}{,}}}

\newcommand{\cA}{\mathcal{A}}

\newcommand{\cD}{\mathcal{D}}

\newcommand{\cO}{\mathcal{O}}
\newcommand{\tcO}{\widetilde{\cO}}

\newcommand{\cP}{\mathcal{P}}

\renewcommand{\epsilon}{\varepsilon}
\renewcommand{\hat}{\widehat}

\renewcommand{\bar}{\overline}

\newcommand{\nothere}[1]{}

\usepackage{xspace}

\newcommand{\DelUCB}{\normalfont \texttt{PatientBandits}\xspace}
\newcommand{\DUCB}{\normalfont \texttt{D-UCB}\xspace}
\newcommand{\aDelUCB}{\normalfont \texttt{Adapt-PatientBandits}\xspace}

\makeatletter
\newcommand{\AlignFootnote}[1]{%
    \ifmeasuring@
    \else
        \footnote{#1}%
    \fi
}
\makeatother

\title{Stochastic bandits with arm-dependent delays}
\usepackage{hyperref}
\newcommand{\footremember}[2]{%
	\footnote{#2}
	\newcounter{#1}
	\setcounter{#1}{\value{footnote}}%
}
 
\title{Stochastic bandits with arm-dependent delays}
\author{%
	Anne Gael Manegueu\footremember{alley}{Otto-von-Guericke University of Magdeburg, Institute for mathematical stochastics, PF 4120, D-39016 Magdeburg, Germany, anne.manegueu@ovgu.de}%
	\and Claire Vernade\footremember{trailer}{GoogleDeepMind, London, UK, vernade@google.com}%
	\and Alexandra Carpentier\footremember{alley1}{Otto-von-Guericke University of Magdeburg, Institute for mathematical stochastics, PF 4120, D-39016 Magdeburg, Germany, alexandra.carpentier@ovgu.de}%
	\and Michal Valko\footremember{trailer1}{GoogleDeepMind, Paris, FR, valkom@deepmind.com}%
}
\date{}
\begin{document}
\maketitle
\begin{abstract}
Significant work has been recently dedicated to the \textit{stochastic delayed bandit setting} because of its relevance in applications. The applicability of existing algorithms is however restricted by the fact that strong assumptions are often made on the delay distributions, such as full observability,  restrictive shape constraints, or uniformity over arms. In this work, we  weaken them significantly and only assume that there is a bound on the tail of the delay. In particular, we cover the important case where the delay distributions vary across arms, and the case where the delays are heavy-tailed. Addressing these difficulties, we propose a simple but efficient UCB-based algorithm called the $\DelUCB$. We provide both problems-dependent and problems-independent bounds on the regret as well as performance lower bounds. 
\end{abstract}
\vspace{-1em}
\section{Introduction}
\label{sec:introduction}
In  realistic applications of \textit{reinforcement learning} (RL), rewards come \textit{delayed}.
In a game, for instance, the consequences of the agent's actions are only observed at the end.
This issue at the core of challenges in RL \citep{garcia1966learning}, when the horizon is finite or geometrically discounted. Even much simpler (state-less) \textit{bandit} setups, such as online advertising suffer from delayed feedback \citep{chapelle2011empirical,chapelle2014modeling}. In particular, most systems do not optimize for clicks but for \emph{conversions}, which are events implying a stronger commitment from the customer. However, different ads trigger different customer response time. Typically, more expensive---and rewarding---products require more time to convert on customers' side, and the system needs to be tuned to be robust to the delays. 


As a result, we study \emph{stochastic delayed bandits}  for which the delay distributions are \textit{arm-dependent} and possibly \textit{heavy-tailed}. 
We consider the realistic model for delayed conversions of \citep{chapelle2014modeling} and \citep{vernade2017stochastic} in which delays are only \emph{partially observable}. Conversions are \textit{binary} events that represent a strong commitment (buying, subscribing, \dots). If a conversion happens, it is sent with some delay to the learner who observes both its reward and the corresponding delay. Otherwise, the reward is null by default but it has no specific delay, only the current waiting time. This models a typical e-commerce application: if a customer does not buy the recommended product, the recommendation system will not be informed. 
The nature of this setup brings two main challenges (1) the \emph{censoring} due to partially observed delays, which forces the learner to deal with an unknown amount of missing feedback; and (2) the \emph{identifiability}\footnote{Ex.: Consider two instances for : (1) reward follows a ${\rm Bernoulli}(1)$ and delay is a {\rm Dirac} in $+\infty$ and (2) reward follows a ${\rm Bernoulli}(0$) and delay is a {\rm Dirac} in 0. Both instances produce the same data but have strictly different parameters.}
issue due to arm-dependent delays.

Prior work for delayed bandits have bypassed the challenges above by assuming that the delays are \textit{observed} \citep{joulani2013online,dudik2011efficient}, which removes the ambiguity, or \textit{bounded by a fixed quantity}~\citep{pike2018bandits, garg2019stochastic, cesa2018nonstochastic}, which gives other possibilities to deal with them. Another approach that has been proposed by \citep{vernade2017stochastic} is to drop the artificial requirement of observability of delays, and instead impose that all delays have the \textit{same} distribution across arms and that this distribution is \textit{known}. 
We further discuss the relevant related work in Section~\ref{sec:Related Work}. While the known approaches yield good results under their strong assumptions on delays, none of them provides a solution to the realistic problem that we are tackling.
\vspace{-1em}
\paragraph{Contributions} This work is the first to consider a stochastic bandit setting with arm-dependent, unbounded, and possibly heavy-tailed delays with partially observable delays. We jointly address the challenges of \citep{vernade2017stochastic}, \citep{zhou2019learning}  and \citep{thune2019nonstochastic}. Unlike \citep{vernade2017stochastic, zhou2019learning}, we make only mild assumptions on the delays. Furthermore, we give a precise characterization of the impact of the delays on the regret than that given in the more difficult, non-stochastic setting of \citep{thune2019nonstochastic}.  Our algorithmic soltuion is $\DelUCB$, the right calibration of upper confidence bounds and prove that it attains \textit{problem-dependent and minimax regret upper bounds}. In particular, we prove that: 
 
\begin{itemize}
    \item In the \textit{asymptotic} regime, the presence of delays does not affect the regret by more than a constant factor with respect to what is achieved in standard bandits. In other words, the loss of information due to the delays \textit{does not lead to a significant increase} of the regret with respect to standards bandits Our algorithm attains the problem-dependent upper bound of the standard bandits up to a constant multiplicative factor in many cases, e.g.\,in the homoscedastic Gaussian case. 
    \item On the other hand, we prove that there is a \textit{drop in performance} with respect to problem-independent guarantees as compared to standard bandits. This is \textit{unavoidable} and we prove a lower bound to support it.
    \item Finally, we study the \textit{impact of imperfect prior knowledge} for $\DelUCB$. Our algorithm takes a parameter that is related to an upper bound on the heaviness of the tails of the delay distributions. We provide a comprehensive study in which respect the precise knowledge of this parameter can be avoided.
\end{itemize}

\section{Bandits with delayed feedback}
\label{sec:setting}


We define our \emph{stochastic delayed bandit} setting. Consider a sequential game of $T\in \mathbb N^\star$ rounds where an agent is interacting with an environment characterized by a finite set of $K \in \mathbb N^\star$ arms which we denote $[K]\triangleq\{1,...,K\}$. An~instance is characterized by a tuple $((\mathcal{V}_i,\cD_i)_{i\in[K]})$, where each arm $i\in [K]$ is associated with both
\begin{itemize}
\item an unknown \textit{reward} distribution $\mathcal{V}_i$ whose support is in $[0,1]$, and with mean $\mu_{i}$,
\item and an unknown \textit{delay} distribution $\mathcal{D}_i$ with cumulative distribution function (CDF) $\tau_{i}$ and support in $\mathbb N$, such that for any $d\geq 0$, $t \leq T$, if $D_t \sim \mathcal{D}_i$, then we have that $\mathbb{P}(D_t \leq d)=\tau_i(d).$
\end{itemize}
At each round $t \leq T$, the learner chooses (pulls) an arm $I_t\in [K]$. A reward $C_t\sim \mathcal{V}_{I_t}$ and a delay $D_t \sim \mathcal{D}_{I_t}$ are generated \textit{independently from each other.} Neither the reward nor the delay is necessarily displayed 
at the current round $t$.  However, at each upcoming round $t+u$ for $1\leq u\leq T-t$, the learner observes the updated quantity 
\begin{equation}\label{eq:rew}
X_{t, u} \triangleq C_t \mathbf 1\{ D_ {t} \leq u\},
\end{equation}
corresponding to her pull at time $t$. 
Note that, conversely, at time $t$, the learner only observes the updated quantities corresponding to its past actions: $(X_{s, t-s})_{s\leq t} \triangleq (C_s \mathbf 1\{ D_ {s} \leq t-s\})_{s\leq t}$. And therefore at round $t$, it disposes of the entire history information \begin{equation}\label{eq:hist}
\mathcal H_t \triangleq (X_{u,v})_{u < t,v \leq t-u}.
\end{equation}
This setting is summarized in Figure~\ref{fig:setting}.

\begin{figure}[t!]
	\fbox{
		\begin{minipage}{0.45\textwidth}
\noindent \textbf{Setting}: $K$ arms, horizon $T$, reward distributions $(\mathcal{V}_i)_{i\leq K}$, delay distributions $(\mathcal{D}_i)_{i \leq K}$\\
\noindent\textbf{for} $t=1$ \textbf{ to} $T$
\begin{itemize}
\vspace{-3mm} 
\setlength\itemsep{-1.5mm}
    \item learner observes updated reward sequence $(X_{s,t-s})_{s \leq t}$, see Eq.\,\ref{eq:rew}
\item  learner chooses $I_t\in [K]$ based on $\mathcal H_t$, see Eq.\,\ref{eq:hist}
\item reward $C_t\sim \mathcal{V}_{I_t}$ and delay $D_t \sim \mathcal{D}_{I_t}$ are generated independently but not necessarily displayed
\vspace{-2mm} 
\end{itemize}
\noindent \textbf{end for} 
		\end{minipage}}
\caption{Delayed learning setting} \label{fig:setting}
\end{figure}

 Note that delays are only \emph{partially observable}: at round $t$ and for some $s\leq t$, if the learner observes $X_{s, t-s} = 0$, there is an \textit{ambiguity}. Either the reward $C_s$ is actually indeed~$0$, or the delay is not yet passed, i.e., $t-s <D_s$. 
 This ambiguity is due to the \textit{multiplicative noise} induced by the delays. Indeed, conditionally on the action taken at time $u<t$, $I_u\in [K]$, the expected observable payoff at round $t$ is scaled by some delay and action dependent factor
\[\mathbb E[X_{u,t-u}\,|\,I_u=i] = \tau_{i}(t-u) \mu_i.\]
In other words, the delays induce temporarily missing data among the observations, but the learner cannot know exactly \emph{how much feedback is missing}. 

Indeed, the heavier the tail of the delay distribution of an arm, the longer it takes for the learner to be able estimate its mean well. This creates dramatic \textit{identifiability} issues: if the best arm is more delayed than the others, its apparent value might seem lower for a while and only a learner that is patient enough shall rightfully identify it as the optimal action. 
To mitigate this issue and to give a chance to a learner to tune its patience level, we rely on the following assumption.
\begin{assumption}[$\alpha$-polynomial tails for the delay distributions]\label{as:alpha} Let $\alpha > 0$ be some fixed quantity. We assume that $\forall m \in {\mathbb N}^\ast$ and $\forall i \in \{1, \ldots, K\}$, it holds that
	\[| 1-\tau_{i}(m) | \leq m^{-\alpha}.\] 
\end{assumption}
The smaller $\alpha$, the more heavy-tailed the delay distribution, and the more difficult the setting. This assumption needs to hold uniformly across arms but does not impose they all have the same distribution, unlike required by \citep{vernade2017stochastic}. This is an important weakening of the restricted setting of the prior work, which we generalize.

For $i\in [K]$, we denote by $T_i(t) \triangleq  \sum _{i=1}^{t}\mathbb{I}\{I_s=i\}$ the number of times that the arm $i$ has been drawn up to round~$t$. As $\mu^\star \triangleq  \max_i \mu_i$ denotes the mean of the best arm(s),  $\Delta_i \triangleq  \mu^\star- \mu_{i}$ is the gap between the mean of the optimal arm(s) and the mean of arm $i$. The goal of the agent is to maximize its expected cumulative reward (i.e., $E[\sum _{t=1}^{T} C_t]$) after $T$ rounds and therefore to minimize the expected regret,
\begin{equation}
\label{eq:regret}
 	\overline{R}_T= T \mu^{\ast}- \mathbb{E}\sum _{t=1}^{T}C_{t} = \sum _{i=1}^{K} \Delta_i  \mathbb{E}[ T_i(T)].
\end{equation} 
\begin{remark}
In this paper, we consider the same concept of regret as for standard bandits, unlike what is done by~\citep{vernade2017stochastic}. We believe it is a more relevant approach that allows for comparison with the vast existing prior work on the topic (see Section~\ref{sec:Related Work}).
\end{remark}

\section{Related work} 
\label{sec:Related Work}

The problem of learning with delayed feedback is ubiquitous in a wide range of applications, including universal portfolios in finance \citep{cover2011universal}, online advertising \citep{chapelle2014modeling} and e-commerce \citep{yoshikawa2018nonparametric}. Therefore, there is a large body of theoretical results designed under different scenarios and assumptions on the delays. We review the contributions of prior works distinguishing between the full information setting and the bandit setting.
 
  \paragraph{Full-information} In online (convex) optimization, as opposed to our setting, the learner must perform gradient descent by estimating the gradient on the fly using the available information. Thus, delayed feedback forces the learner to make decisions in the face of additional uncertainty. This problem has been considered by \citep{weinberger2002delayed} in the cadre of prediction of individual sequences under the assumption that the delays were fixed and known. The study of online gradient type of algorithms under possibly random or adversarial delays is made under various hypotheses by \citep{langford2009slow, quanrud2015online, joulani2016delay}. In distributed learning, communication time between servers naturally induces delays, and this particular setting was studied by \citep{agarwal2011distributed,mcmahan2014delay, sra2015adadelay} who all proposed asynchronous or delay-tolerant versions of \textsc{AdaGrad}. And an attempt to reduce the impact of delays was made by \citep{mann2018learning} by allowing the learner to observe intermediate signals. 
  
  \paragraph{Bandits with observed delays. } \citep{joulani2013online} provide a clear overview of the impact of the delays for both the stochastic and adversarial setting. Using a method based on a non-delayed algorithm called Base, he will succeed in extending the work of \citep{weinberger2002delayed} to the non-constant delays case. Following up in this work, \citep{mandel2015queue} argued in favor of more randomization to improve exploration in delayed environments, although the guarantees remained unchanged. Taking a different path, closer to ours, the recent work of \citep{zhou2019learning} relies on a \emph{biased estimator} of the mean and corrects for it in the UCB using an estimator of the amount of missing information. They consider the contextual bandit setting and make a strong assumption on the delay distribution, imposing that 1) delays are fully observable, and 2) the delay distribution is the same for all arms and should concentrate nicely (bounded expectation). Due to these assumptions, their algorithm cannot be used in our setting and cannot be compared to ours. \citep{thune2019nonstochastic} considered the adversarial bandit setting, where delays are observed right after (or before) sampling an arm. Unfortunately their results and algorithms \textit{do not} apply in our setting since we \textit{do not} observe the delays before or right after sampling an arm. 

  \paragraph{Bandits with \textit{partially} observed delays. } The delayed bandits with censored observations was introduced by  \citep{vernade2017stochastic}, who builds on the real-data analysis of \citep{chapelle2011empirical}. They rely on the major assumption that delays are the same across arms and have a finite expectation.
  In this setting they prove an asymptotic problem-dependent lower bound that recovers the standard Lai \& Robbins' lower bound. They propose an algorithm that uses as input the CDF of the delay distribution and matches this asymptotic lower bound.  In other words, they prove that asymptotically, well-behaved delays have no impact on the regret. 
  Following up in this work, \citep{vernade2018contextual, arya2019randomized} extend this setting to the linear and contextual stochastic setting. Two  other papers consider the case where the delays are not observed at all - but are bounded by a constant $D>0$. \citep{garg2019stochastic} analyze the stochastic setting, and \citep{cesa2018nonstochastic} the adversarial setting and achieve a regret of order $\sqrt{T K\log K}+ K D \log T$ and $\sqrt{DTK}$ respectively. \citep{pike2018bandits} when further considering unbounded delays in adversarial setting but time under the assumption that only their expectation is bounded. Again these results do not apply in our context as we do not assume that the delays are bounded - and under Assumption~\ref{as:alpha} with $\alpha<1$, the delays can even have infinite means.

  

\section{The $\DelUCB$ algorithm}
\label{sec:algorithm}

In this section we describe an optimistic algorithm \citep{auer2002finite} that is able to cope with partially observed and potentially heavy-tailed delays. The $\DelUCB$ algorithm estimates high-probability upper confidence bounds on the parameter of each arm. 
As opposed to the standard UCB approach, it is hopeless to design conditionally unbiased estimators in this delayed setting, so the algorithm also needs to properly bound the bias for each arm adaptively. 
Throughout the paper, we use the notation $A\land B := \min (A,B)$ and $A\lor B:= \max (A,B)$.
\paragraph{A delay-corrected, high-probability UCB.} 
Delays being partially observable, the learner must build its estimators \emph{with an unknown number of observations}. Indeed, since rewards are delayed, a certain proportion of the feedback of each arm is missing but it is impossible to know exactly how much because the zeros are ambiguous. Nonetheless, we show that it is possible to prove high-probability confidence bounds for the parameters of the problem, provided that we correctly handle this extra bias due to the delays. For this purpose, we rely on Assumption~\ref{as:alpha}, that gives us a loose global bound on the tails of the distributions of the delays. 

At a time $t \geqslant K+1$, given the history of pulls and observed rewards $\mathcal{H}_t$, we define the mean estimator:
\begin{align}
\label{eq:mu_hat}
 \hat{\mu}_{i}(t)&= \frac{1}{T_{i}(t)} \sum _{u=1}^{t} X_{u,t-u}\mathbf 1 \{I_u=i\},
\end{align}
where $T_{i}(t) =  \sum _{u=1}^{t} \mathbf 1 \{I_u=i\}$. 
The key ingredient for our algorithm is the upper confidence bound. The following theorem is our first major contribution and provides the required high-probability bound. 

\begin{theorem}
 \label{th:boundmu}
 Let $i \in [K]$ and $\alpha >0$ satisfy Assumption~\ref{as:alpha}.
 Then for any $t> K$ and $\delta>0$, with probability $1-\delta$
 \begin{align}
 \label{eq:ucb}
| \hat{\mu}_{i}(t) -  \mu_i|     
&\leq \left(\frac{2 \log \frac{2} {\delta}} {T_i(t)}\right)^{1/2} \!\!\! + 2 T_i(t)^{- (\alpha \land 1/2)}.
\end{align}
\end{theorem}

\begin{proof} The full proof of this result is provided in Appendix~\ref{ap:confi_inter}. It relies on the following decomposition
\begin{align*}
 | \hat{\mu}_{i}(t) -  \mu_i| \leq \Big| \hat{\mu}_{i}(t) - \frac{1}{T_i(t)}\sum_{u=1}^{t} \tau_i(t-u) \mu_i  \mathbf 1\{I_u = i\} \Big|\\
 + \Big|\frac{1}{T_i(t)}\sum_{u=1}^{t} \tau_i(t-u) \mu_i  \mathbf 1\{I_u = i\}-\mu_i\Big|.
 \end{align*}
 On the right-hand side, the first term is a usual deviation term. The probability that it's larger than $\left(\frac{2\log(2\delta)}{T_t(t)}\right)^{1/2}$ is uniformly bounded by $\delta/2$. The second term corresponds to the bias and is bounded by $2 T_i(t)^{- \alpha \land 1/2},$  which comes from simply summing the $\tau_i(t-u)$ in the worst case, i.e., when all pulls of arm $i$ are made in the last $T_i(t)$ rounds.
\end{proof}
A clear benefit of the above result is a simple and easy-to-compute \textit{adaptive} upper bound on the parameter $\mu_i$ for our estimator. This UCB is similar to the standard UCB2 \citep{auer2002finite} except for the \textit{extra bias term} that goes to zero with the number of pulls. 
In fact, it adaptively trades off bias and variance as a function of $\alpha$: it is the largest for small values of $\alpha \leq 1/2$, that is when delays have very large tails. Indeed, $\alpha$ plays an important role in our algorithm presented in details below. 
\paragraph{$\DelUCB$ is described in Algorithm~\ref{fig:algo}.}
\begin{algorithm}
\caption{$\DelUCB$ \label{fig:algo}}
\begin{algorithmic}
\STATE \textbf{Input:}: $\alpha >0$, horizon $T$, number of arms $K$.
\STATE \textbf{Initialisation:}  Pull each arm once and set for all $i \in [K]$: $T_i(t)=1$ and initialise $\hat{\mu}_i(t)$ according to Eq.\,\ref{eq:mu_hat}.
\FOR{$t=K+1...T$}
\STATE Pull arm $I_t \in \argmax_{i \in [K]} UCB_i(t)$ 
\vspace{1mm} 
\STATE Observe all feedback updates $(X_{s,t-s})_{s\leq t}$
\ENDFOR
\end{algorithmic}
\end{algorithm}
 It receives as input the parameter $\alpha > 0$, the horizon $T$, and the number of arms $K$ which we assume to be smaller than~$T$. In the first phase of the game, all arms are pulled once. The player then pulls the arm from $[K]$ that has the highest UCB as defined in Theorem~\ref{th:boundmu},
\[
UCB_i(t) = \hat{\mu}_{i}(t) + \left(\frac{2 \log(2KT^3)} {T_i(t)}\right)^{1/2} \!\! + 2 T_i(t)^{-(\alpha \land 1/2) }.
\]
The algorithm then pulls an arm $I_t$ that maximises $UCB_i(t)$. 

 \section{Analysis of $\DelUCB$}
 \label{sec:theory}

We present the analysis of $\DelUCB$. We also provide a non-asymptotic lower-bound for delayed bandits. 
We provide first the following problem-dependent upper bound on the regret of $\DelUCB$. Its proof is deferred to Appendix~\ref{ap:regret_bound} and follows the lines of the usual analysis of UCB by \citep{auer2002finite}, see also \citep[\S 7]{lattimore2019bandit}.
\begin{theorem}\label{thm:pbdep} Let $T> K \geq 1 $ and $\alpha>0$. Let $((\mathcal{V}_i, \cD_i)_{i\in[K]})$ be the problem as defined in Section~\ref{sec:setting} such that Assumption~\ref{as:alpha} holds. If  $\DelUCB$ is run with parameters $\cP = (\alpha, T,K)$, it achieves
\begin{align*}
\overline{R}_T &\leq \sum _{i: \Delta_i>0}  \Big[\frac{64 \log(2T)}{\Delta_i}  \lor \Big(\frac{8}{\Delta_i}\Big)^{\frac{1-\alpha}{\alpha} \lor 1}\Big] + 2K.
\end{align*}
\end{theorem}
The only term that depends on $T$ in Theorem~\ref{thm:pbdep} is of the order of $\sum _{i: \Delta_i>0} \log(T)/\Delta_i$. It is of the same order as the classical bound for UCB which is asymptotically optimal, see~\citep{lai1985asymptotically}. Note that this was expected, since \citep{vernade2017stochastic} showed that delays should not have an asymptotic impact on the regret\footnote{Their result is stated for delays with finite expectation but remains valid in our setting.}. Our bound has an additional term of order $\sum _{i: \Delta_i>0} \Big(\frac{8}{\Delta_i}\Big)^{\frac{1-\alpha}{\alpha} \lor 1}$. This term does not depend on $T$, so it is asymptotically negligible. But if $\alpha<1/2$ and some of the gaps are very small, it can be large from a non-asymptotic perspective - see Section~\ref{sec:discussion} where we discuss this further. 



We now provide a problem-independent upper bound.
\begin{theorem}\label{thm:pbindep}
 Let $T> K \geq 1 $ and $\alpha>0$. If $\DelUCB$ is run with parameters $\cP=(\alpha, T,K)$, for any \emph{stochastic delayed problem} such that Assumption~\ref{as:alpha} holds, it achieves
\begin{align*}
\overline{R}_T &\leq  2\times 64^{(1-\alpha) \lor 1/2} T^{1 - \alpha \land 1/2}\Big(K\log(2T) \Big)^{\alpha \land 1/2} +2K.
\end{align*}
\end{theorem}
Up to logarithmic terms and multiplicative constants, the order of magnitude of this bound is 
$\max\Big(\sqrt{KT},  K^\alpha T^{1-\alpha}\Big)$. Whenever $\alpha\geq 1/2$, the order of the bound is $\sqrt{KT}$ - as is the case for UCB (up to logarithmic terms) and for more refined algorithms like MOSS in~\cite{audibert2009minimax} \textit{in the classical stochastic bandit setting (without delays)}. However if the delays are allowed to be more heavy-tailed, i.e., $\alpha<1/2$, then the regret starts degrading with $\alpha$ as the upper bound is of order $ K^\alpha T^{1-\alpha}$. 
We prove that this degradation of the (problem-independent) regret is unavoidable.
\begin{theorem}\label{thm:lbreg}
Consider $K = 2$ and $T\geq K$, and $\alpha>0$. There exists a Bernoulli stochastic delayed bandit problem satisfying Assumption~\ref{as:alpha}, such that the expected regret of any algorithm $\bar R_T$ on this problem is larger than $T^{1 - \alpha}/8.$
\end{theorem}
Combining this theorem with the classical problem independent lower bound in classical stochastic bandits --- see e.g., the book of \citep{lattimore2018bandit} 
---- one obtains that the order of magnitude of the worst case regret (for bandit problems satisfying Assumption~\ref{as:alpha}) of any algorithm is larger than $\max\Big(\sqrt{KT},   T^{1-\alpha}\Big)$. 
This matches (up to logarithmic terms) the upper bound in Theorem~\ref{thm:pbindep} \textit{with respect to $T$} (not to $K$ whenever $\alpha<1/2$).

\section{Adaptation to $\alpha$}\label{sec:alpha}

$\DelUCB$ requires (a lower bound on ) $\alpha$ as input. 
It is indeed natural to ask whether this prior information on the delays is necessary. In other words, can we design an algorithm that learns $\alpha$ as well or adapts to the delays on-the-fly ? And how much would the regret be impacted ? 
In this section we give a detailed answer to those questions, both in the asymptotic and non-asymptotic regime. In the latter, we prove a negative result in the general case. However, we propose a new assumption under which adaptivity is achievable.

\subsection{Adaptation of the problem dependent regret to $\alpha$.}

We first study possibilities of adaptation in the asymptotic regime. An immediate corollary of Theorem~\ref{thm:pbdep} is as follows.
\begin{corollary}\label{cor:pbdep} Let $T> K \geq 1$ and consider a bandit problem with minimum gap $\bar \Delta = \min_{k: \Delta_k>0} \Delta_k$, and where each arms $k$ satisfies Assumption~\ref{as:alpha} for $\alpha_k$. Consider $T>e^{e^e}$ large enough so that a) $\log\log(T)/\log(T) \leq \min_i \alpha_i$, b) $8 \bar \Delta^{-1} \leq \log T$. If $\DelUCB$  is run with parameters $((\log\log(t)/\log(t))_{t\leq T}, T,K)$, it achieves:
\begin{align*}
\overline{R}_T &\leq \sum _{i: \Delta_i>0}  \Big[\frac{128 \log(2T)}{\Delta_i} \lor \Big(\frac{8}{\Delta_i}\Big)^{\frac{1-\alpha}{\alpha} \lor 1}\Big] + 2K.
\end{align*}
\end{corollary}\vspace{-2mm}
So for $T$ large enough depending on problem dependent quantities, it is possible to run a slight variant of $\DelUCB$ that takes as input the sequence $(\alpha_t=\log\log(t)/\log(t))_{t\geq 1}$ instead of a fixed $\overline{\alpha}$. 
So, for a fixed problem, asymptotically, knowing $\alpha$ is not necessary.



\subsection{Impossibility result under Assumption~\ref{as:alpha} for adapting the problem independent regret to $\alpha$}
For a fixed horizon $T<\infty$, however, it is a different story. Our second lower bound below  states that if you give a `too small' input parameter $\alpha$ to a `good' algorithm, then it has a suboptimal regret, even in the simpler case where $K=2$. 

Specifically, we define the class of \textit{$\alpha$-optimal algorithms $\cA_{\alpha}$} as the algorithms whose expected regret 
is smaller than $T^{1-\alpha}/8$ for all bandit instances satisfying Assumption~\ref{as:alpha} for a fixed~$\alpha$.

\begin{theorem}\label{thm:lbreg2}
Consider $K = 2,$ $T\geq K,$ and fix $\alpha >0$. 
For any $\beta \geq  \alpha,$ there exists a Bernoulli bandit instance satisfying Assumption~\ref{as:alpha} for $\beta$ such that the expected regret $\bar R_T$ of any $\alpha$-optimal algorithm is larger than $T^{1 -  \alpha}/8>T^{1-\beta}/8.$
\end{theorem}
In other words, an algorithm that performs optimally uniformly under Assumption~\ref{as:alpha} for a given $\alpha$  \textit{cannot at all} adapt to $\beta \geq \alpha$.

\subsection{New algorithm for adapting the problem independent regret to $\alpha$ under more restrictive assumptions}

Yet, is adaptivity a lost cause? Under the weak Assumption~\ref{as:alpha}, Theorem~\ref{thm:lbreg2} above is quite disheartening. A structural reason for this is that it is impossible to estimate $\alpha$ under this assumption. We show that under a \emph{slightly} more restrictive assumption this becomes possible.

\begin{assumption}\label{as:alph2}
Assume that there exists  $0<c\leq 1$ and $\bar \mu >0$, such that $\min_k \mu_k > \bar \mu$ and for all $i\in [K]$,
\begin{equation}
cm^{-\alpha} \leq |1 - \tau_i(m)| \leq m^{-\alpha}.   
\end{equation}
Assume also that $\alpha \geq \underline{\alpha}$ for some $\underline{\alpha}>0$.
\end{assumption}


This assumption \emph{does not mean} that the delay distributions of the arms are all the same. It means that the parameter $\alpha$ now globally characterizes the tails of the delay distributions. The challenge for estimating it is that delays are only \emph{partially observable}. To further explain Assumption~\ref{as:alph2}, let's consider small and a large delay d and D, such that $D>d>0$. The conditional expectation of the difference of the same reward after respectively $d$ and $D$ time steps have passed is:
\begin{align}
\mathbb E_{|I_t}[ X_{t,D} - X_{t,d}] &= \mu_{I_t} \tau_{I_t}(D) - \mu_{I_t} \tau_{I_t}(d)\nonumber\\
&\in [c\mu_{I_t}d^{-\alpha} - \mu_{I_t}D^{-\alpha}, \mu_{I_t}d^{-\alpha}],\nonumber
\end{align}
where $\mathbb E_{|I_t}$ is the conditional expectation with respect to the arm $I_t$ pulled at time $t$, and 
where $c>0$ comes from Assumption~\ref{as:alph2}. Therefore, if $d \leq \left(c/2\right)^{1/\underline{\alpha}}D$, we now have that\footnote{Note that $c\leq 1$ so that with this definition of $d,D$, we have that $D \geq d$.}
\[ \mathbb E_{|I_t}[ X_{t,D} - X_{t,d}] \in [\frac{c \bar\mu}{2}d^{-\alpha} , d^{-\alpha}], \]
where $\bar\mu, \underline{\alpha}>0$ are defined in Assumption~\ref{as:alph2}. So we can now see that it is possible to estimate $\alpha$ up to a logarithmic factor using the logarithm of an estimator of $\mu_{i} \tau_{i}(D) - \mu_{i} \tau_{i}(d)$, if we properly choose $d$ and $D$ from   some arm $i$ sampled often enough.
We now formalize this idea, introducing all the necessary quantities. 

In all this section, , we denote $\bar I_t \triangleq \argmax_k T_k(t)$. We only use the samples of this arm to estimate $\alpha$ at each round. To simplify the notation let $\bar T_t  \triangleq T_{\bar I_t}(t)$ and for some delay $D$, let

\[\bar{m}_{t,D} \triangleq \frac{1}{\bar{T}_{t-D}}\sum_{s=1}^{t-D}X_{s,D}\mathbf 1 \{I_s=\bar I_t\}\] 

be the sample mean after waiting $D$ steps.
We set $D_t \triangleq \lfloor \bar T_t/2\rfloor$ and $d_t \triangleq \left\lfloor\left(c/2)\right)^{1/\underline{\alpha}} D_t\right\rfloor$. Subsequently, we define the estimator of $\alpha$ at round $t$ as 
\[\hat \alpha_t \triangleq \min\left(-\frac{\log\left(\bar m_{t,D_t} - \bar m_{t,d_t} \right)}{\log(\bar T_t)}\CommaBin \frac12 \right)\cdot\]
Such an estimator is related to quantile or CDF-based estimators used in extreme value theory \cite{de2007extreme,carpentier2015adaptive}. 
We define set the lower confidence bound on it as
\[\bar \alpha_t \triangleq \left[\hat \alpha_{t} - \frac{\log\left(\frac{2^4\sqrt{\log(2KT^3)}}{c \bar \mu}\right)}{\log(\bar T_t)}\right]\lor 0.\]

$\aDelUCB$ simply uses $\bar \alpha_t$ for the computation the upper confidence bounds as in Eq.\,\ref{eq:ucb}. The algorithm therefore does not need to know for which parameter $\alpha$ Assumption~\ref{as:alph2}
is satisfied. We summarize \aDelUCB in Algorithm~\ref{fig:algo2}.

\begin{algorithm}
\caption{$\aDelUCB$ \label{fig:algo2}}
\begin{algorithmic}
\STATE \textbf{Input:} $c, \underline{\alpha},  \bar\mu, T,K$
\STATE \textbf{Initialisation:}  Pull each arm twice.
\FOR{$t=2K+1,...,T$}
\STATE Pull the arm  $I_t\in \!\!\displaystyle \argsup_{i \in  \{1, \ldots, K\}} UCB_{i}(t-1)$ when using the parameter $\bar \alpha_t$ in the UCB.
\vspace{1mm} 
\STATE Observe all individual feedback $(X_{s,t-s})_{ s\leq t}$.
\ENDFOR
\end{algorithmic}
\end{algorithm}
 The expected regret of the \aDelUCB is bounded by the following theorem.

\begin{theorem}\label{thm:pbindepadapt}
 Let $T> K \geq 1 $ and $\alpha, \underline{\alpha}, c,  \bar\mu >0, $ such that Assumption~\ref{as:alph2} holds. The  
 expected regret of $\aDelUCB$ is bounded as 
 \[
 R_T = \tcO(K^\alpha T^{1-\alpha \land \nicefrac{1}{2}}).
 \]
\end{theorem}

And so $\aDelUCB$ achieves a regret that depends
on the unknown parameter $\alpha$ and is of the same order as the upper bound of Theorem~\ref{thm:pbindep} up to logarithmic term. This algorithm is therefore minimax optimal up to logarithmic terms relatively to the lower bound defined in Section~\ref{sec:theory}.

\section{Discussion}
\label{sec:discussion}

   \label{subsec:discussion}
   \paragraph{Comparison to bandits without delays.} 
   As discussed in the Section~\ref{sec:algorithm}, the major difference between $\DelUCB$ and the classical UCB algorithm is the extra bias term. In the classical bandit setting, we have strictly more information than in our setting. When rewards are delayed, the algorithm always has to deal with temporarily missing data: some actions have been taken but their rewards are missing until the delay has passed.  In particular, when $\alpha <1$, the delays are heavy-tailed, so their expectation is infinite, and this buffer of missing data always keeps growing in size with $T$, creating a non-negligible bias in the estimators. This difference is even more important when $\alpha < 1/2$, which corresponds to the situation where the bias term is larger than the usual deviation term. For such a problem, it is clear that a classical UCB would have a linear regret, see Section~\ref{sec:experiments}.
   
We now discuss optimality of $\DelUCB$ by commenting both our problem-dependent (Th.~\ref{thm:pbdep}) and problem-independent (Th.~\ref{thm:pbindep}) bounds. 
   

   \begin{itemize}
       \item The problem-dependent bound for our problem is of the order $\sum_{k: \Delta_k >0} \log(T)/\Delta_k$. Up to a term that depends only on the $(\Delta_k)_k$ (and not on $T$, see Section~\ref{sec:theory}), this is of the same order than the problem-dependent bound in classical stochastic bandits~\cite{lai1985asymptotically} - and it does \textit{not} depend on the delay distributions (e.g.~it does not depend on $\alpha$).  
       \item On the other hand, the problem-independent bound, is of order $T^{1 - \alpha \land (1/2)} K^{\alpha \land (1/2)}$ up to logarithmic terms. It differs widely from the problem-independent bound from classical stochastic bandits, which is of order $\sqrt{KT}$ -- see e.g.~\citep[Chapter~7]{lattimore2019bandit}. It is of same order when $\alpha \geq 1/2$, and is larger when $\alpha <1/2$. However,  Theorem~\ref{thm:lbreg} ensures that this rate is minimax optimal with respect to $T$, up to a multiplicative term $K^{\alpha \land (1/2)}$ and some logarithmic terms. This means that this gap is the price to pay for having delays that are potentially long, and can therefore pose strong bias problems.
   \end{itemize}
   \vspace{-0.5cm}
   \paragraph{Parameters of the algorithms.} 
   Our algorithm needs as input a parameter $\alpha$, which is the only external prior information on the delays given as input.  It is a more delicate question to decide whether this prior information is necessary. We discuss it extensively in Section~\ref{sec:alpha}. In a nutshell, from an asymptotic perspective, the knowledge  of $\alpha$ is not necessary, but from a non-asymptotic perspective it is. Nonetheless, under a slightly stronger hypothesis on the delay distribution, see Assumption~\ref{as:alph2}, there exists a fully adaptive algorithm Adapt-$\DelUCB$. Its regret is sublinear, see Theorem~\ref{thm:pbindepadapt}, and matches that of $\DelUCB$ up to a logarithmic term.

\section{Experiments}
\label{sec:experiments}
In this section, we evaluate the empirical performance of $\DelUCB$. Throughout this section, we will provide experiments where the delays of each arm $i$ follows the Pareto Type I distribution with tail index $ \alpha_i$, and where the rewards of each arm $i$ follows a Bernoulli distribution with parameter $\mu_i$. 

First, we investigate the performances of $\DelUCB$ with respect to the parameters of the problem, $(\alpha_i,\Delta_i)_{i\in[K]}$, and to the hyperparameter of the algorithm $\alpha$ -  which we write $\overline{\alpha}$ here to avoid confusion. 

Second, we compare it to the state of art baseline, which is the censored version of \DUCB of \citep{vernade2017stochastic} with various threshold windows. 
This algorithm takes two parameters, the threshold $m$ and the CDF $\tau$ of the delays\footnote{\citep{vernade2017stochastic} assumed that the delay distributions are \textit{know and homogeneous across arms}.}. The threshold $m$ calibrates the time that the algorithm waits before updating reward. To the best of our knowledge, \DUCB  is the strongest baseline that deals with \textit{partially observed delays}.





\subsection{Influence of the parameters of the problem}


\paragraph{Study of the hyperparameter $\overline{\alpha}$.} $\DelUCB$ takes $\overline{\alpha}$ as a parameter. The choice of this parameter is a key point for the implementation of  $\DelUCB$. 
Ideally we would like to take $\overline{\alpha} = \min_i\alpha_i$ but in the absence of information on the delay distributions we cannot do this. We therefore illustrate the sensitivity of our method to the mis-calibration of $\overline{\alpha}$. We consider a $2$-arm setting with horizon $T= 3000$, with arm means $\mu=(0.5, 0.55)$ and with tail index $\alpha_1=1, \alpha_2=0.3$ respectively. We consider $\overline{\alpha}\in[0.02,0.5]$ and display in Figure~\ref{exp:alpha1} the regret in function of $\overline{\alpha}$.
\begin{figure}[h]
\begin{center}
\includegraphics[width=0.4\textwidth]{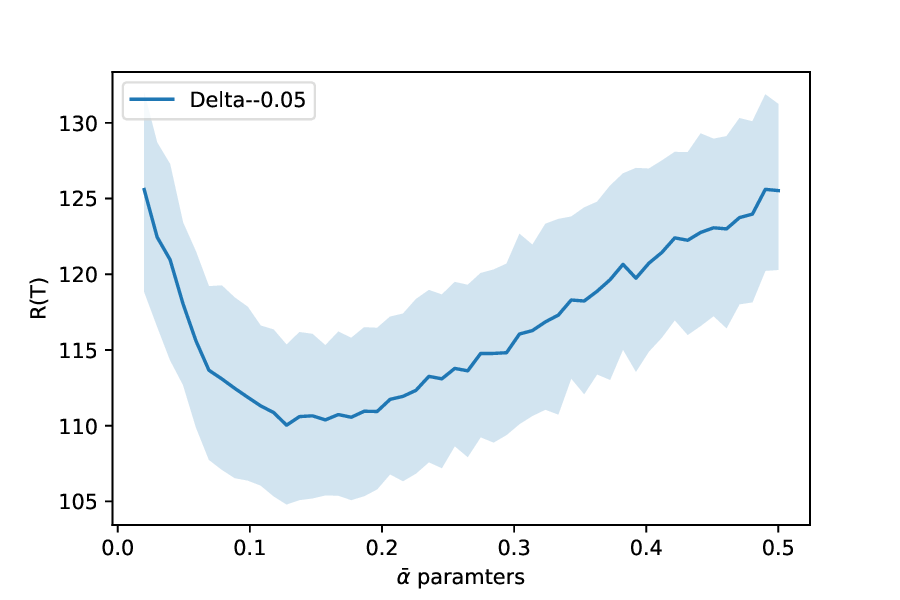}
\caption{Regret at round $T=3000$ of $\DelUCB$ in function of $\overline{\alpha}\in[0.02,0.5]$ for the bandit problem $\mu=(0.5, 0.55)$, and $\alpha_1 = 1$ and $\alpha_2 = 0.3)$. Results are averaged over 400 runs.
}\label{exp:alpha1}
\vspace{-0.5cm}
\end{center}
\end{figure}
It can be seen that the regret first decreases with $\overline{\alpha}$ and then increases after approximately $\overline{\alpha} = 0.3$. This is precisely what is expected For small $\overline{\alpha}$ the algorithm is consistent but explores too much and this induces a large regret. For $\overline{\alpha}$ larger than $\min_i \alpha_i$, the regret starts to increase again, since the bias coming from the delays are not sufficiently taken into account by the UCB. 
 
\paragraph{Study of the impact of $(\alpha_i, \Delta_i)_{i\in [K]}$.} We now investigate the dependency of the regret of $\DelUCB$ on the delay parameters and arm gaps  $(\alpha_i, \Delta_i)_{i\in [K]}$. We consider the following two armed problems where we set $\mu=(0.4, 0.4+\Delta)$ where we take $\Delta\in [0.02,...,0.6]$ - fixing the horizon to $T= 3000$. For each problem, we respectively choose $\alpha_1=1$ and $\alpha_2\in\{0.2, 0.3,0.4, 0.5, 0.8\}$ and run the $\DelUCB$ policy with optimal parameter $\overline{\alpha} = \alpha_2$, so that we can see the impact of $\alpha_2$ and $\Delta$ independently from calibration issues. The results represented in the Figure ~\ref{exp:alpha2} display the influence of the arm gap $\Delta$ on the regret, for various values of $\alpha_2$.
\begin{figure}[h]
\begin{center}
\includegraphics[width=0.4\textwidth]{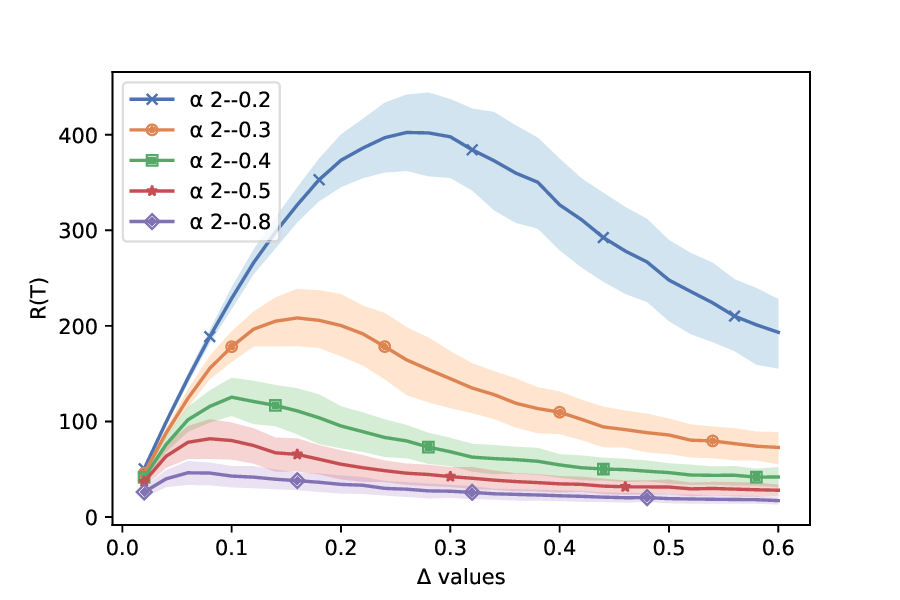}
\caption{Regret of $\DelUCB$ in function of the arm gap $\Delta \in [0.02,0.6]$ where the bandit problem is characterized by $\mu=(0.5, 0.5+\Delta)$ and $\alpha_1=1$ and where $\alpha_2\in\{0.2, 0.3,0.4, 0.5, 0.8\}$ - each curve corresponds to a different value of $\alpha_2$. For each problem, $\DelUCB$ was run with horizon $T=3000$ and with parameter $\overline{\alpha} = \alpha_2$ and the results are averaged over 300 runs.}
 \label{exp:alpha2}\vspace{-0.5cm}
\end{center}
\end{figure}

Figure~\ref{exp:alpha2} illustrates a standard phenomenon in stochastic bandits, and which also holds for delayed bandits: for small values of the arm gap $\Delta$, the regret increases with $\Delta$. This corresponds to the fact that for small $\Delta$, the algorithm explores and is not able to focus on the most promising arms since the arms means are too close. Then at some point for larger values of $\Delta$ the regret starts decreasing, as predicted by the bound in Theorem~\ref{thm:pbdep}. A phenomenon that is specific to delayed bandits is that the smaller $\alpha_2$, the larger the regret. This is expected from Theorem~\ref{thm:pbindep}, since the smaller $\alpha_2$, the more delayed the rewards, and the harder the problem. A more subtle phenomenon, also illustrating Theorem~\ref{thm:pbindep}, is that the smaller $\alpha_2$, the larger the value and the position of the maximum of each curve - the maximum or the regret being bounded by the problem independent bound that depends here on $\overline{\alpha} = \alpha_2$.


\subsection{Comparing with \textmd{D-UCB}. }

 We compare here the regret of $\DelUCB$ with the one of \textmd{D-UCB} as a function of the horizon $T$, for different values of the parameters - respectively $\overline{\alpha}$ and $m$. Since \textmd{D-UCB} was designed for the context where the distribution of the delays is the same for all arms, i.e.~$\alpha_i$ identical over arms, we consider that scenario as well as the more general case with heterogeneous $\alpha_i$'s.
 
  \paragraph{Homogeneous delay distributions across arms.} In the first scenario, we consider a two armed bandit problem with means $\mu=(0.6,0.8)$. We set the same tail index for both arms, i.e.~$\alpha_1=\alpha_2=0.7$. We run $\DelUCB$ for $\overline{\alpha}\in \{0.1, 0.5\}$. For \textmd{D-UCB} we consider various threshold parameters  $m\in \{10, 50, 100, 200\}$, and feed \textmd{D-UCB} with the \textit{exact delay distribution of the arms} - which gives \textmd{D-UCB} an important edge over our algorithm. The results are displayed in Figure~\ref{exp:comp3} (regret as a function of time). 
\begin{figure}[h]
\begin{center}
\includegraphics[width=0.4\textwidth]{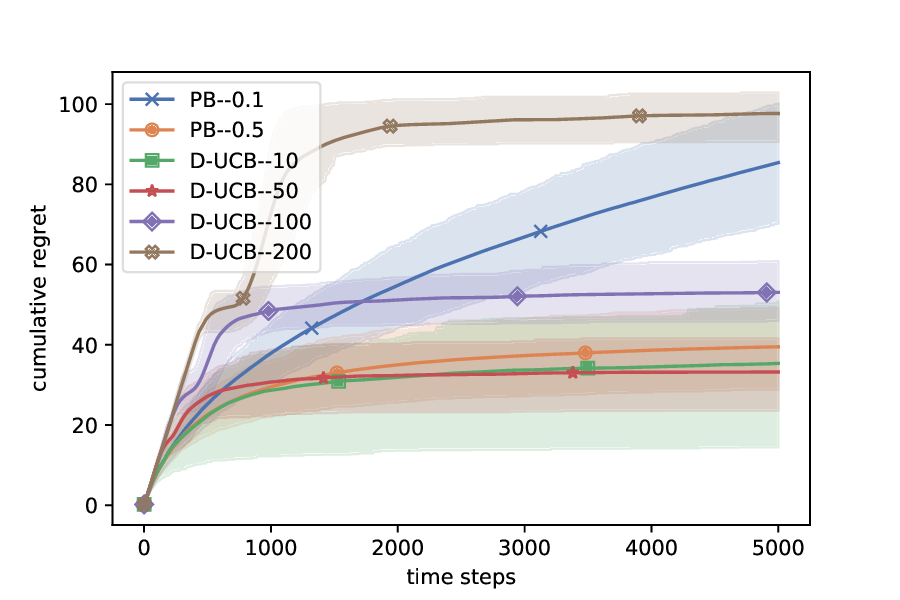}
 \caption{Regret of the \textmd{D-UCB} and $\DelUCB$ for $\mu=(0.6, 0.8)$ and with homogeneous delay distributions characterised by $\alpha_1 =  \alpha_2 = 0.7$. We plot results for $\DelUCB$ with parameters $\overline{\alpha}=(0.1,0.5)$, and for \textmd{D-UCB} with parameters $m=(10, 50, 100, 200)$. The results are averaged over 400 runs.}
 \label{exp:comp3}
 \vspace{-7mm}
\end{center}
\end{figure}

The performances of \textmd{D-UCB} and $\DelUCB$ are comparable, in particular in the case of good calibration of the parameters, i.e.~respectively $\overline{\alpha} = 0.5$, and $m=50$. $\DelUCB$ performs slightly worse in the best case, but note that \textmd{D-UCB} is tuned with the full knowledge of the CDF of the delays. 
An observation coming from Figure~\ref{exp:comp3} is the presence of long lasting linear phases at the initial stage of learning of \textmd{D-UCB} for large $m$ which tend to be caught up over time. This comes from the structure of the algorithm, and from the fact that it has to wait until $m+K$ time steps before it starts exploiting the observations - which is not the case for our strategy.
\paragraph{Non-homogeneous delay distributions.} In the second scenario we still consider a two-armed bandit problem with means $\mu=(0.6,0.8)$, and we set the parameters of the tail distribution of the delays as $\alpha_1=1, \alpha_2=0.3$. This is a 'difficult' scenario, since arm $2$ which has the highest mean has also the lowest delay parameter. This means that its delays are more heavy tailed. We consider as before $\DelUCB$ with parameters $\overline{\alpha}\in \{0.1, 0.5\}$, and the \textmd{D-UCB} for threshold parameters  $m\in \{10, 50, 100, 200\}$. Regarding the CDF parameter of \textmd{D-UCB}, we provide a Pareto distribution with parameter $0.7$. 
The results for all policies are displayed on Figure~\ref{exp:comp1} (regret in function of horizon $T$). 
\begin{figure}[h!]
\begin{center}
\includegraphics[width=0.4\textwidth]{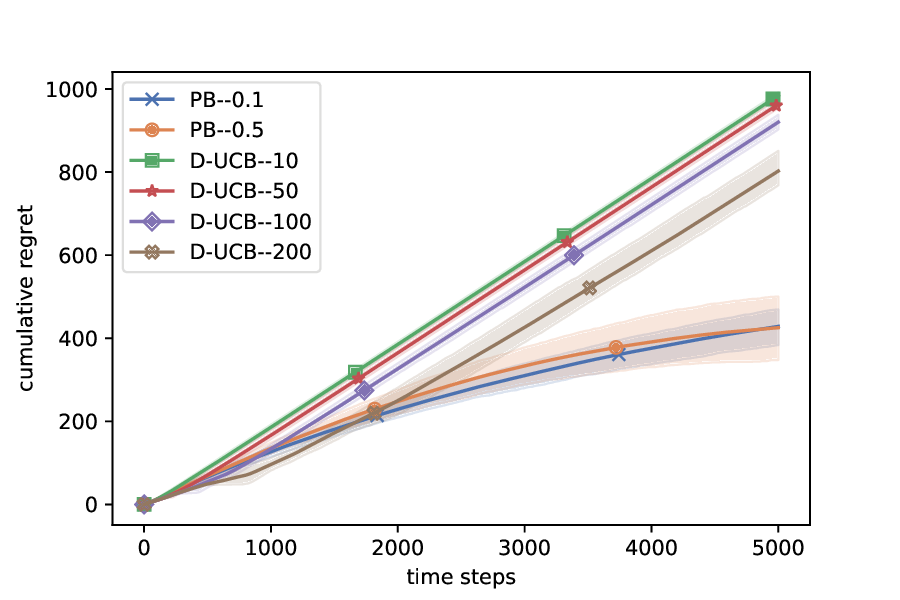}
 \caption{Regret of the \textmd{D-UCB} and $\DelUCB$ for $\mu=(0.6, 0.8)$ and with delay distributions that vary across arms, characterized by $\alpha_1 =1$ and $\alpha_2 =  0.3$. We plot results for $\DelUCB$ with parameters $\overline{\alpha}=(0.1,0.5)$, and for \textmd{D-UCB} with parameters $m=(10, 50, 100, 200)$. The results are averaged over 400 runs. 
}\vspace{-7mm}
 \label{exp:comp1}
\end{center}
\end{figure}
We observe that \textmd{D-UCB} has a very high regret, increasing linearly with $T$. This can be explained by the fact that \textmd{D-UCB}  \textit{cannot} adapt to delay distributions varying across arms. In other words, it does not take the heterogeneity of the delays into account and can be confused by this difficult situation where the best arm also corresponds to the longest delays. Consequently, \textmd{D-UCB} focuses only on observations that are substantially biased and  misidentifies the best arm. On the other hand, $\DelUCB$ adapts to the heterogeneous delays and manages to identify the best arm, leading to a sub-linear regret.

\section{Conclusion}
\label{sec:Conclusion}
ßIn this paper, we extend the problem of learning with bandit feedback and partially observable delays to arm-dependent delay distributions with possibly unbounded expectations. We close many existing open problems left by \citep{vernade2017stochastic, vernade2018contextual}, either with positive answers (Theorem~\ref{thm:pbdep}) or negative answers (Theorem~\ref{thm:lbreg2}). The major difficulty faced by the learner in this setting is the identifiability issue due to missing rewards which induce a bias in the estimator of the real payoff. Under the assumption that the tail distribution of the delays is bounded - although it might be very heavy tailed - we designed a very simple UCB-based algorithm, termed $\DelUCB$. We proved that $\DelUCB$ performs almost as well as the standard UCB in the classical, non-delayed case, from a problem dependent point of view. We also studied the problem of adaptivity to the delay distributions and concluded that this is not possible (Theorem~\ref{thm:lbreg2}) unless a global bound on the tails hold (Assumption~\ref{as:alph2}). Closing the gap between the problem-independent bound and the lower bound may constitute the object of future studies.

\paragraph{Acknowledgements.} The work of A. Carpentier is partially supported by the Deutsche Forschungsgemeinschaft (DFG) Emmy Noether grant MuSyAD (CA 1488/1-1), by the DFG - 314838170, GRK 2297 MathCoRe, by the DFG GRK 2433 DAEDALUS (384950143/GRK2433), by the DFG CRC 1294 'Data Assimilation', Project A03, and by the UFA-DFH through the French-German Doktorandenkolleg CDFA 01-18 and by the UFA-DFH through the French-German Doktorandenkolleg CDFA 01-18 and by the SFI Sachsen-Anhalt for the project RE-BCI. The work of A. Manegueu is supported by the Deutsche Forschungsgemeinschaft (DFG) CRC 1294 'Data Assimilation', Project A03.


\appendix
\onecolumn
\section{Proof of Theorems~\ref{th:boundmu}}
\label{ap:confi_inter}
The aim of this section is to bound the deviation of the estimator $\hat \mu_i$ from the true mean $\mu_i$. For this purpose, we first begin by defining the favorable event, that is, the event for which  all confidence intervals hold for all arms at all time steps. We then prove with Hoeffding's inequality that this event occurs with high probability. And finally, we derive the desired results by bounding the bias incurred by $\hat \mu_i$ and leveraging the properties of this favorable event.

 \paragraph{Step 1: the favorable event.} First let us define the following quantities (for $i \leq K$ and $u \leq T_i(T)$):
$$\overline{t}_{i}(u) = \inf \{t\geq 0 :  \sum_{u\leq t} \mathbf 1\{I_u = i\} = u\},$$
$$\overline{C}_{i,u} = C_{\overline{t}_{i}(u)},$$
$$\overline{D}_{i,u} = D_{\overline{t}_{i}(u)}.$$
Here $\overline{t}_{i}(u)$ is the time where we pulled arm $i$ for the $u$-th time, $\overline{C}_{i,u}, \overline{D}_{i,u}$ are respectively the corresponding reward and delay. Note that
\begin{align}\label{eq:defdef}
\sum _{u=1}^{t} X_{u,t-u}\mathbf 1\{I_u = i\} = \sum _{u=1}^{T_i(t)} \overline{C}_{i,u} \mathbf 1\{ \overline{t}_{i}(u) + \overline{D}_{i,u} \leq t \}.
\end{align}

 We define the event $\xi$ as follows;
\[\xi \triangleq \left\{ \forall i \in \{1,..., K \},  \forall t \in \{1,..., T \}, \forall s \in \{1,..., T_i(t) \}:  \left|\sum _{u=1}^{s} \overline{C}_{i,u} \mathbf 1\{ \overline{t}_{i}(u) + \overline{D}_{i,u} \leq t \} - \sum _{u=1}^{s}  \tau_{i}(t - \overline{t}_{i}(u)) \mu_i \right| \leq  \sqrt {2\log \frac{2} {\delta}s }  \right\}\cdot\]
Note that for fixed $i \in \{1,..., K \},  t \in \{1,..., T \}, s \in \{1,..., T_i(t) \}$, we have that
\[\Bigg(\sum_{u \leq v} \Big[\overline{C}_{i,u} \mathbf 1\{ \overline{t}_{i}(u) + \overline{D}_{i,u} \leq t \} -   \tau_{i}(t - \overline{t}_{i}(u)) \mu_i\Big]\Bigg)_{v \leq s}\]
is a martingale adapted to the filtration $\Big(\sigma(\overline{C}_{i,v}, \overline{D}_{i,v}, \overline{t}_{i}(v))\Big)_{v \leq s}$. And since the martingale increments $\Big[\overline{C}_{i,u} \mathbf 1\{ \overline{t}_{i}(u) + \overline{D}_{i,u} \leq t \} -   \tau_{i}(t - \overline{t}_{i}(u)) \mu_i\Big]$ belong to $[-1,1]$ by assumption, it holds by Azuma-Hoeffding's inequality that with probability larger than $1-\delta$
\[\Bigg|\sum_{u \leq s} \Big[\overline{C}_{i,u} \mathbf 1\{ \overline{t}_{i}(u) + \overline{D}_{i,u} \leq t \} -   \tau_{i}(t - \overline{t}_{i}(u)) \mu_i\Big]\Bigg| \leq \sqrt {2\log \frac{2} {\delta}s }.\]
Since $T_i(t) \leq t$, it holds by a union bound that
\begin{align}\mathbb P(\xi) \geq 1 - KT^2 \delta.\label{prob:xi}\end{align}


 \paragraph{Step 2: Bound on the $|\hat \mu_i(t) - \mu_i|$ on $\xi$.}

By Equation~\ref{eq:defdef}, it holds that
\begin{align}\xi \subset \left\{\forall i \in \{1,..., K \},  \forall t \in \{1,..., T \}, \left|\hat \mu_i(t) - \frac{1}{T_i(t)}\sum_{u=1}^{t} \tau_{i}(t-u) \mu_i \mathbf 1\{I_u = i\}\right| \leq  \sqrt {\frac{2\log \frac{2} {\delta}}{ T_i(t)} } \right\}.\label{eq:xi}\end{align}

Note that we have  
\begin{align*}
\Big|\frac{1}{T_i(t)}\sum_{u=1}^{t} \tau_{i}(t-u) \mu_i  \mathbf 1\{I_u = i\}  -  \mu_i \Big| &\leq \frac{1}{T_i(t)}\sum_{u=1}^{t} |(1 - \tau_{i}(t-u)) \mu_i  \mathbf 1\{I_u = i\}| \\
&\leq   \frac{1}{T_i(t)}\sum_{u=1}^{t}  \mathbf 1\{I_u = i\} ((t-u)\lor 1)^{-\alpha},
\end{align*}
since $0 \leq \mu_i \leq 1$ and since by Assumption~\ref{as:alpha} it holds that $|\tau_{i}(m) - 1| \leq (m\lor 1)^{-\alpha}$. And since $\sum_{u=1}^{t}  \mathbf 1\{I_u = i\}  = T_i(t)$, we have
\begin{align}
\Big|\frac{1}{T_i(t)}\sum_{u=1}^{t} \tau_{i}(t-u) \mu_i  \mathbf 1\{I_u = i\}  - T_i(t) \mu_i \Big| 
&\leq   \frac{1}{T_i(t)}\sum_{v=0}^{T_i(t)}  (v \lor 1)^{-\alpha} \nonumber\\
&\leq   \frac{1}{T_i(t)}\left[2+ \int_{1}^{T_i(t)}  v^{-\alpha} dv\right], 
\label{eq:bia}
\end{align}
whenever $\alpha \leq 1/2$:
\[
 \frac{1}{T_i(t)}\Big[2+ \int_{1}^{T_i(t)}  v^{-\alpha} dv\Big] =  \frac{1}{T_i(t)}\Big[2+ [\frac{-1}{1- \alpha}v^{1-\alpha}]_1^{T_i(t)} \Big]
\]
\begin{align*}
  & =  \frac{1}{T_i(t)} \Big[2+ \frac{1}{1- \alpha} [T_i(t)^{1-\alpha} - 1] \Big]\\
& = \frac{1}{T_i(t)}\Big[ -\frac{2\alpha}{1- \alpha} + \frac{1}{1- \alpha}T_i(t)^{1-\alpha} \Big]\\
&\leq \frac{1}{T_i(t)}\Big[ 2 T_i(t)^{1-\alpha} \Big]\\
&\leq 2 T_i(t)^{-\alpha}  = 2 T_i(t)^{-\alpha \land (1/2)}.
\end{align*}

Now for $\alpha \geq 1/2$:
\begin{align*}
  & \frac{1}{T_i(t)}\Big[2+ \int_{1}^{T_i(t)}  v^{-\alpha} dv\Big] \leq  \frac{1}{T_i(t)}\Big[2+ \int_{1}^{T_i(t)}  v^{-1/2} dv\Big]\\
& =  \frac{1}{T_i(t)}\Big[2+ 2 [T_i(t)^{1/2} - 1] \Big] = 2T_i(t)^{1/2} = 2 T_i(t)^{-\alpha \land (1/2)}. \\
\end{align*}

Note that we have
\begin{align}
| \hat{\mu}_{i}(t) -  \mu_i| &= \Big| \hat{\mu}_{i}(t) - \frac{1}{T_i(t)}\sum_{u=1}^{t} \tau_{t-u} \mu_i  \mathbf 1\{I_u = i\} + \frac{1}{T_i(t)}\sum_{u=1}^{t} \tau_{t-u} \mu_i  \mathbf 1\{I_u = i\}-\mu_i\Big|\nonumber\\
&\leq  \Big| \hat{\mu}_{i}(t) - \frac{1}{T_i(t)}\sum_{u=1}^{t} \tau_{t-u} \mu_i  \mathbf 1\{I_u = i\} \Big|+ \Big|\frac{1}{T_i(t)}\sum_{u=1}^{t} \tau_{t-u} \mu_i  \mathbf 1\{I_u = i\}-\mu_i\Big|.\nonumber
\end{align}
Now from the definition of the favorable event $\xi$ in Equation~\eqref{eq:xi} and from the bound of the bias in Equation~\eqref{eq:bia}, it holds on $\xi$ that:
\begin{align}
| \hat{\mu}_{i}(t) -  \mu_i|     
&\leq \left(\frac{2 \log \frac{2} {\delta}} {T_i(t)}\right)^{1/2} + 2 T_i(t)^{- \alpha \land 1/2}.\label{eq:boundmu}
\end{align}



\vspace{5mm}

\section{Proof of Theorems~\ref{thm:pbdep} and~\ref{thm:pbindep}}
\label{ap:regret_bound}

The proof of these theorems relies on the results obtained in Appendix~\ref{ap:confi_inter}. Indeed, we will use the deviation results obtained on the event $\xi$ to bound the number of pulls of sub-optimal arms in order to derive the upper bounds on the regret.

\paragraph{Step 1: upper bound on the number of pulls of sub-optimal arms.}  Let's assume that at some given time $t+1>K$ the algorithm pulls a sub-optimal arm $i$ (such that $\mu_i < \mu^*$). According to the algorithm's rules, we have: $UCB_i(t+1) \geq UCB_{k^{\ast}}(t+1)$. And so on $\xi$, we have because of Equation~\eqref{eq:boundmu}
$$ \mu^{\ast}  \leq  UCB_{k^{\ast}}(t+1) \leq UCB_i(t+1) \leq \mu_i + 2\left(\frac{2 \log \frac{2} {\delta}} {T_i(t)}\right)^{1/2} + 4 T_i(t)^{- \alpha \land 1/2}.$$ 
Rearranging the terms, we have on $\xi$:
\begin{align} 
 \Delta_i = \mu^{\ast} - \mu_i &\leq 2\left(\frac{2 \log \frac{2} {\delta}} {T_i(t)}\right)^{1/2} + 4 T_i(t)^{- \alpha \land 1/2}, \nonumber
\end{align}
which implies that on $\xi$
$$T_i(t) \leq \frac{16 \log(2/\delta)}{\Delta_i^2} \lor \Big(\frac{8}{\Delta_i}\Big)^{-\frac{1}{\alpha} \lor 2} \lor 1,$$
and so on $\xi$, we have for any sub-optimal arm $i$
\begin{align}\label{eq:boundpull}
T_i(T) \leq \frac{16 \log(2/\delta)}{\Delta_i^2} \lor \Big(\frac{8}{\Delta_i}\Big)^{\frac{1}{\alpha} \lor 2} \lor 1.
\end{align}

\paragraph{Step 2: Conclusion.} Consider a sub-optimal arm $i$ (such that $\mu_i < \mu^*$). Combining Equation\eqref{eq:boundpull} with Equation~\eqref{prob:xi}, and since $T_i(T) \leq T$ we have
\begin{align*}
\mathbb{E} [ T_i(T)] &\leq \frac{16 \log(2/\delta)}{\Delta_i^2} \lor \Big(\frac{8}{\Delta_i}\Big)^{\frac{1}{\alpha} \lor 2} \lor 1 + KT^{3}\delta\\
&\leq \frac{16 \log(2KT^3)}{\Delta_i^2} \lor \Big(\frac{8}{\Delta_i}\Big)^{\frac{1}{\alpha} \lor 2} \lor 1 + 1,
\end{align*}
for $\delta \triangleq (KT^3)^{-1}$. Let us now consider some value $\Delta>0$. We have by definition of the regret and with this $\delta\triangleq  (KT^3)^{-1}$ as above:
\begin{align}\label{eq:mix}
\overline{R}_T &\leq \sum_{i: \Delta_i>\Delta} \Delta_i \Big[\frac{16 \log(2KT^3)}{\Delta_i^2} \lor \Big(\frac{8}{\Delta_i}\Big)^{\frac{1}{\alpha} \lor 2} \lor 1 + 1\Big] + \Delta \sum_{i: \Delta_i \leq \Delta} \mathbb{E} [ T_i(T)].
\end{align}
Taking $\Delta = 0$ and recalling that $K \leq T$, we obtain the result of Theorem~\ref{thm:pbdep}, namely
\begin{align*}
\overline{R}_T &\leq \sum _{i: \Delta_i>0}  \Big[\frac{64 \log(2T)}{\Delta_i} \lor \Big(\frac{8}{\Delta_i}\Big)^{\frac{1-\alpha}{\alpha} \lor 1}\Big] + 2K.
\end{align*}
Now not that since the function $\Delta_i \Big[\frac{16 \log(2KT^3)}{\Delta_i^2} \lor \Big(\frac{8}{\Delta_i}\Big)^{\frac{1}{\alpha} \lor 2} \lor 1 + 1\Big]$ increases when $\Delta_i$ decreases, we have for any $\Delta>0$
\begin{align*}
\overline{R}_T &\leq  K \Big[\frac{16 \log(2KT^3)}{\Delta} \lor \Big(\frac{8}{\Delta}\Big)^{\frac{1 - \alpha}{\alpha} \lor 1} \Big] + \Delta T + 2K\\
&\leq K\log(2T) \Big(\frac{64}{\Delta}\Big)^{\frac{1 - \alpha}{\alpha} \lor 1} + \Delta T + 2K.
\end{align*}
And so for $\Delta = \Big(\frac{K\log(2T) 64^{\frac{1 - \alpha}{\alpha} \lor 1}}{T}\Big)^{\alpha \land 1/2}$, we have
\begin{align*}
\overline{R}_T &\leq  2\times 64^{(1-\alpha) \lor 1/2} T^{1 - \alpha \land 1/2}\Big(K\log(2T) \Big)^{\alpha \land 1/2} +2K,
\end{align*}
which concludes the proof of Theorem~\ref{thm:pbindep}.

\section{Proof of Theorem~\ref{thm:lbreg}}

In order to prove this result, we need to show that there exists a bandit problem in the family described in Section~\ref{sec:setting} such that the regret at $T$ is $\Omega (T^{1-\alpha})$. 
To do so, we construct two problems in that family and show that for at least one of them, the regret is larger than the desired quantity. 

First, for ease of notation, define 
\begin{equation}
    \label{eq:p_q}
    p= T^{-\alpha} \quad \text{and} \quad q=\frac{p}{4-2p}.
\end{equation}

We construct two alternative problems with two arms, $K=2$. In both cases, we fix arm 1 such that the distribution of the rewards $\nu_1$ is $\mathcal B(1/2)$ and the distribution of the delays $\mathcal D_1$ is $\delta_0$ (i.e.~a Dirac mass in $0$, no delays). 
\begin{itemize}
    \item Problem A: $\nu_2^{(A)}$ is $\mathcal B(1/2-q)$ and $\mathcal D_2^{(A)}$ is $\delta_0$.
    \item Problem B: $\nu_2^{(B)}$ is $\mathcal B(1/2+q)$ and $\mathcal D_2^{(B)}$ is $(1-p)\delta_0 + p \delta_{T}$.
\end{itemize}
In Problem~A, arm 1 is the best with gap $\Delta=q$ and there are no delays. 
In Problem~B, arm 2 is the best, with gap $\Delta=q$ too, but delays are sending a proportion $p$ of the rewards to $t\geq T$ so they cannot be used for learning. Thus, the conditional distribution of $X_{s,u}| I_s = 2$ is in fact $\mathcal B((1/2 + q)(1-p))$. 
Note that 
\[
\left(\frac{1}{2} + q\right)(1-p) = \frac{1}{2} - \left(\frac{p}{2} -q +pq\right) 
= \frac{1}{2} - \left(\frac{p}{2} - \frac{p}{4-2p} +\frac{p^2}{4-2p}\right) = \frac{1}{2} - \frac{p}{4-2p} = \frac{1}{2} -q.
\]

So the effective mean of arm 2 is $1/2 -q$, meaning that arm 2 has the same distribution in both problems. This implies in particular that
\[
\mathbb E_A T_2(T) = \mathbb E_B T_2(T),
\]
where $\mathbb E_a$ is the expectation in problem $a\in \{A,B\}$. 

And so if we write $\bar R^{(a)}_T$ for the regret in scenario $a \in \{A,B\}$:
$$\max_{a \in \{1,2\}} \bar R_T^{(a)} \geq q \max( T - \mathbb E_B T_2(T),  \mathbb E_A T_2(T)) \geq qT/2.$$
This concludes the proof as $q =  p/(4-2p) \geq p/4 = T^{-\alpha}/4$.

\section{Proof of Theorem~\ref{thm:lbreg2}}

The proof of this theorem relies on the same tools as for Theorem~\ref{thm:lbreg} above, but uses a slightly different reasoning.
Namely, we now fix $\alpha>0$ and we restrict the family of algorithms to those that have a regret smaller than $T^{1-\alpha}/8$ for any stochastic bandit problem satisfying Assumption~\ref{as:alpha} for $\alpha$. We denote this family $\cA_{\alpha}$.

We want to prove that there exists a bandit problem satisfying Assumption~\ref{as:alpha} for some $\alpha' > \alpha$ such that any algorithm in $\cA_{\alpha}$ has regret at least $T^{1-\alpha}/8 > T^{1-\alpha'}/8$. This proves that any algorithm minimax optimal for $\alpha$ is suboptimal for $\alpha'>\alpha$.

Similarly to the previous section, fix $p = T^{\alpha}$ and $q=p/(4-2p)$ as in Eq.~\eqref{eq:p_q}, and consider the two problems,
\begin{itemize}
    \item Problem A: $\nu_2^{(A)}$ is $\mathcal B(1/2-q)$ and $\mathcal D_2^{(A)}$ is $\delta_0$.
    \item Problem B: $\nu_2^{(B)}$ is $\mathcal B(1/2+q)$ and $\mathcal D_2^{(B)}$ is $(1-p)\delta_0 + p \delta_{T}$.
\end{itemize}
Note that, Problem B satisfies Assumption~\ref{as:alpha} for $\alpha$, while Problem A satisfies it for any $\alpha' >0$ so in particular for $\alpha'>\alpha$. 
So for any algorithm in $\cA_{\alpha}$, $q \mathbb E_B [T_1(t)]<3T^{1-\alpha}/16$.
But, as we proved above, because of the delays, the algorithm cannot distinguish both problems and we have 
$\mathbb E_A T_1(T) = \mathbb E_B T_1(T) $, i.e.~the average number of pulls of arm $1$ is the same in both problems. 
Thus, 
\[ 
\mathbb E_A T_1(T) = \mathbb E_B [T_1(t)]<  q^{-1}  T^{1-\alpha}/8.
\]

Using that $q>p/4$ as before, 
\[
\max_{a \in \{A,B\}} \bar R_T^{(a)} \geq  q(T-\mathbb E_A T_1(T))   > \frac{T^{-\alpha}}{4} T-T^{1-\alpha}/8 = T^{1-\alpha}/8.
\]
\section{Proof of Theorem~\ref{thm:pbindepadapt}}

We start by stating the full version of the theorem that 
guarantees the performance of \aDelUCB.
\begin{theorem}
 Let $T> K \geq 1 $ and $\alpha, \underline \alpha, c, \bar\mu >0, $ such that Assumption~\ref{as:alph2} holds. The regret of \aDelUCB is bounded as
\[
\bar R_T \leq   8^{17} \left(\frac{1}{c \bar \mu}\right)^{4}  \left(\frac{2}{c}\right)^{4(\alpha \land 1/2)/\underline{\alpha}} \log(2T)^{13/2} T  (K/T)^{\alpha \land (1/2)}.
\]
\end{theorem}
Before proving Thereom~\ref{thm:pbindepadapt}, we first provide the following proposition that bounds the error on our estimator of $\alpha$.
\begin{proposition}\label{prop:alph}
Let $\delta \in (0,1)$. There exists an event of probability larger than $1 - 2KT^2\delta$ such that for any $t \leq T,$
\[\alpha\land 1/2-\frac{\log\left(2^3\left(\left(\frac{2}{c}\right)^{(\alpha\land 1/2)/\underline{\alpha}}+B\right)\right) }{\log(\bar T_t)}   \leq \hat \alpha_t \leq  \alpha \land 1/2 + \frac{\log\left(\frac{2^{7/2} B}{c \bar \mu}\right)}{\log \bar T_t}\CommaBin\]
where $B \triangleq B_{\delta} \triangleq \sqrt{2\log(2/\delta)}.$
\end{proposition}


\begin{proof}[Proof of Proposition~\ref{prop:alph}]
For a lighter notation we set $\bar \mu_t \triangleq \mu_{\bar I_t}$. Similarly to the analysis of the subset of $\xi$ in Equation~\ref{eq:xi}, we can prove that the event 
\[\xi' \triangleq \left\{ \forall t \leq T, \; |\bar m_{t,D_t} -  \bar \mu_t \tau_{\bar I_t}(D_t)| \leq \sqrt{\frac{2\log(2/\delta)}{\bar T(t - D_t)}}\CommaBin \; |\bar m_{t,d_t} -  \bar \mu_t \tau_{\bar I_t}(d_t)| \leq \sqrt{\frac{2\log(2/\delta)}{\bar T(t - d_t)}}\right\}\]
has probability larger than $1-2KT^2\delta$. Let us set 
\[B\triangleq B_{\delta} \triangleq \sqrt{2\log(2/\delta)}.\]
Since $d_t \leq D_t$, by Assumption~\ref{as:alph2} we have that on $\xi'\!,$
\[c\bar \mu_t d_t^{-\alpha} -  \bar \mu_t D_t^{-\alpha}- \frac{2B}{\sqrt{\bar T(t - D_t)}} \leq \bar m_{t,D_t} - \bar m_{t,d_t} \leq     \bar \mu_t d_t^{-\alpha}+\frac{2B}{\sqrt{\bar T(t - D_t)}}\cdot\]
We now chose $d_t \triangleq \left\lfloor\left(\frac{c}{2}\right)^{1/\underline{\alpha}} D_t\right\rfloor,$ for which we have that $c\bar \mu_t d_t^{-\alpha} -  \bar \mu_t D_t^{-\alpha} \geq \frac{c}{2}\bar \mu_t d_t^{-\alpha}$ and therefore on $\xi',$
\[\frac{c}{2}\bar \mu_t d_t^{-\alpha}- \frac{2B}{\sqrt{\bar T(t - D_t)}} \leq \bar m_{t,D_t} - \bar m_{t,d_t} \leq     \bar \mu_t d_t^{-\alpha}+\frac{2B}{\sqrt{\bar T(t - D_t)}}\cdot\]
Here we have that $D_t = \lfloor \bar T_t/2 \rfloor$, and since $\bar T_t \geq 2$, we obtain
$$\bar T_t^{-\alpha} \leq D_t^{-\alpha} \leq 2^{2\alpha}\bar T_t^{-\alpha}.$$
Moreover, since we chose $d_t = \left\lfloor\left(\frac{c}{2}\right)^{1/\underline{\alpha}} D_t\right\rfloor,$ we infer that
\[T_t^{-\alpha} \leq d_t^{-\alpha} \leq 2^{3\alpha}\left(\frac{2}{c}\right)^{\alpha/\underline{\alpha}}\bar T_t^{-\alpha}.\]

Therefore, since $\bar T(t - D_t) \leq \bar T_t \leq 2\bar T(t - D_t),$ we have that on $\xi'$, 
\[\frac{c}{2} \bar \mu \bar T_t^{-\alpha}- 2^{3/2} B \bar T_t^{-1/2} \leq \bar m_{t,D_t} - \bar m_{t,d_t} \leq \left(2^{3(\alpha \land 1/2)}\left(\frac{2}{c}\right)^{(\alpha\land 1/2)/\underline{\alpha}}+2^{3/2}B\right)  \bar T_t^{-\alpha \land 1/2} \triangleq C_\alpha \bar T_t^{-\alpha \land 1/2},\]
where we use the fact that the $\mu_k$ are in $[\bar \mu,1]$. 
\paragraph{First case --- small $\alpha$} First, consider the case where $\frac{c}{4} \bar \mu \bar T_t^{-\alpha} \geq  2^{3/2} B \bar T_t^{-1/2}\!.$ Then we have that on event $\xi'$
\[\frac{c}{4} \bar \mu \bar T_t^{-\alpha} \leq \bar m_{t,D_t} - \bar m_{t,d_t} \leq C_\alpha\bar T_t^{-\alpha \land 1/2},\]
which implies that on $\xi',$
\[-\log\left(\frac{c \bar \mu}{4}\right) + \alpha\log(\bar T_t) \geq -\log(\bar m_{t,D_t} - \bar m_{t,d_t}) \geq -\log\left(C_\alpha\right) + (\alpha \land 1/2) \log(\bar T_t).\]
Therefore, on $\xi',$
\[\alpha-\frac{\log(\frac{c\bar \mu}{4})}{\log(\bar T_t)} \geq -\frac{\log(\bar m_{t,D_t} - \bar m_{t,d_t})}{\log(\bar T_t)}  \geq \alpha\land 1/2-\frac{\log\left(C_\alpha\right)}{\log(\bar T_t)}\CommaBin\]
from which we finally get that on $\xi',$
\begin{equation}\label{eq:fcase}
    \alpha \land 1/2 + \frac{\log(\frac{4}{c\bar \mu})}{\log(\bar T_t)}\geq \hat \alpha_t \geq \alpha\land 1/2-\frac{\log\left(C_\alpha\right)}{\log(\bar T_t)}\cdot
\end{equation}
 Note that while the left hand side of the above inequality \textit{only true under the assumption of the first case} that demands $\frac{c}{4} \bar \mu \bar T_t^{-\alpha} \geq  2^{3/2} B \bar T_t^{-1/2}\!,$ the right hand side is also true when this assumption does not hold since we did not use it.

\paragraph{Second case --- large $\alpha$} Now consider the case where $\frac{c}{4} \bar \mu \bar T_t^{-\alpha} \leq  2^{3/2} B \bar T_t^{-1/2}$. In this case it holds that
\[\bar T_t^{\alpha-1/2} \geq  \frac{c \bar \mu}{2^{7/2} B}  \triangleq  b^{-1},\]
which means that
\[\alpha-1/2 \geq  - \frac{\log\left(b\right)}{\log \bar T_t}\CommaBin\]
and therefore,
\[\alpha \land 1/2 \geq 1/2 - \frac{\log\left(b\right)}{\log \bar T_t}\cdot\]
Now by definition of $ \hat \alpha_t,$ we have
\[\hat \alpha_t\leq 1/2 \leq  \alpha \land 1/2 + \frac{\log\left(b\right)}{\log \bar T_t}\cdot\]
Taking only the right hand side of the first case in Equation~\ref{eq:fcase}, which does \textit{not} use the assumption of the first case, unlike the left hand side (cf.\,the remark under Equation~\ref{eq:fcase}), we have that on $\xi,$
\[ \hat \alpha_t \geq \alpha\land 1/2-\frac{\log\left(C_\alpha\right)}{\log(\bar T_t)}\cdot\]
combining the two sides of the bound, we get that on $\xi,$
\[\alpha \land 1/2 + \frac{\log\left(\frac{2^{7/2} B}{c \bar \mu}\right)}{\log \bar T_t} \geq \hat \alpha_t \geq \alpha\land 1/2-\frac{\log\left(C_\alpha\right) }{\log(\bar T_t)}\cdot\]
Notice that this inequality holds also in the first case (small $\alpha$) since $4/(c\bar \mu) \leq \frac{2^{7/2} B}{c \bar \mu}$. The final result follows immediately since
we have that $C_\alpha \leq 2^3\left(\left(\frac{2}{c}\right)^{(\alpha\land 1/2)/\underline{\alpha}}+B\right)$.
\end{proof}

Now leveraging these concentration bounds obtained on the error of the estimation of $\alpha$ as well as the theoretical results from Theorem~\ref{thm:pbindep} we prove the main result.

\begin{proof}[Proof of Theorem~\ref{thm:pbindepadapt}]
Recall that, from Proposition~\ref{prop:alph} with probability larger than $1 - 2KT^2\delta$ we have that,
\[\alpha \land 1/2 + \frac{\log\left(\frac{2^{7/2} B}{c \bar \mu}\right)}{\log \bar T_t} \geq \hat \alpha_t \geq \alpha\land 1/2-\frac{\log\left(2^3\left(\left(\frac{2}{c}\right)^{(\alpha\land 1/2)/\underline{\alpha}}+B\right)\right) }{\log(\bar T_t)}\CommaBin\] where
$B \triangleq B_{\delta} = \sqrt{2\log(2/\delta)}.$\\
Let $L_\alpha\triangleq \log\left(\frac{2^{7/2}}{c \bar \mu}\right) + \log\left(2^3\left(\left(\frac{2}{c}\right)^{(\alpha\land 1/2)/\underline{\alpha}}+1\right)\right)$. With probability larger than $1 - 2KT^2\delta,$
\[U_t \triangleq \frac{\log\left(\frac{2^{7/2} B}{c \bar \mu}\right) + \log\left(2^3\left(\left(\frac{2}{c}\right)^{(\alpha\land 1/2)/\underline{\alpha}}+B\right)\right)}{\log(\bar T_t)} \leq \frac{L_\alpha + \log(B)}{\log(\bar T_t)}\]
is a high probability lower deviation on the lower bound $\bar \alpha_t$ on $\alpha$. Note that $\bar T_t$ is the number of pulls of the most pulled arm at time t and therefore  $\bar T_t \geq t/K$. Furthermore, note that after the initialisation phase, we have guaranteed that $t \geq 2K$ from which we get
\[U_t \leq \frac{2L_\alpha + 2\log(B)}{\log(t)}\cdot\]

Moreover, by Theorem~\ref{thm:pbindep}, we have that conditionally on the event $\xi$ from Proposition~\ref{prop:alph}, the expected regret coming from the samples pulled \textit{after} round $t$ can be bounded as 
$$128 \sqrt{\log(2T)}T  (K/T)^{\alpha \land (1/2) - U_t }   +2K.$$
The above bound implies that conditional on the event $\xi'$ from Proposition~\ref{prop:alph}, the expected regret due to the samples obtained after round $t = T^{1/2}$ can be bounded as 
$$128 \sqrt{\log(2T)}T  (K/T)^{\alpha \land (1/2) - 4(L_\alpha+\log B)/\log T}    +2K.$$
 Now setting $\delta \triangleq (KT^3)^{-1}$ in the algorithm - where $\delta$ is used to define the event $\xi'$ in  Proposition~\ref{prop:alph} - we have on $\xi'$ that,
$$\bar R_T \leq T^{1/2} + 128 \sqrt{\log(2T)}T  (K/T)^{\alpha \land (1/2) - 4\left(L_\alpha+\sqrt{2}\log\left(2\log(2KT^3)\right)\right)/\log T}  +4K,$$
where the additional regret $T^{1/2}$ comes from the first $T^{1/2}$ samples to the bound computed above, and where the additional regret $2K$ comes from the case when the event $\xi'$ from Proposition~\ref{prop:alph} does not hold - leading to a supplementary term bounded by $\mathbb P((\xi')^c)T \leq  2K$.

Notice that in the expression in the regret bound above, 
\[T^{ 4\left(L_\alpha+\sqrt{2}\log\left(2\log(2KT^3)\right)\right)/\log T} \leq e^{4L_\alpha} \left(8 \log(2T)\right)^{4\sqrt{2}}\leq  8^{12} \left(\frac{1}{c \bar \mu}\right)^{4}  \left(\frac{2}{c}\right)^{4(\alpha \land 1/2)/\underline{\alpha}} \log(2T)^{6}\]
and therefore we can simplify our guarantee to finally obtain
\[\bar R_T \leq T^{1/2} +8^{16} \left(\frac{1}{c \bar \mu}\right)^{4}  \left(\frac{2}{c}\right)^{4(\alpha \land 1/2)/\underline{\alpha}} \log(2T)^{13/2} T  (K/T)^{\alpha \land (1/2)}  +4K.\]

\end{proof}
\end{document}